\documentclass[twoside,leqno,twocolumn]{article}

\usepackage[letterpaper]{geometry}

\usepackage{ltexpprt}
\usepackage[utf8x]{inputenc}
\usepackage{hyperref}
\usepackage{cite}
\usepackage{amsmath,amssymb,amsfonts}
\usepackage{algorithmic}
\usepackage{graphicx}
\usepackage{textcomp}
\usepackage{xcolor}
\def\BibTeX{{\rm B\kern-.05em{\sc i\kern-.025em b}\kern-.08em
    T\kern-.1667em\lower.7ex\hbox{E}\kern-.125emX}}
\usepackage{microtype}
\usepackage{subfigure}
\usepackage{booktabs} 
\usepackage{multirow}

\usepackage{mathtools}

\usepackage{paralist}
\usepackage{graphicx}
\usepackage{subfigure}
\usepackage{paralist}
\usepackage{longtable}
\usepackage{multirow}
\usepackage{booktabs}
\usepackage{tabularx}
\usepackage{mdwlist}
\usepackage{url}
\usepackage{verbatim}
\usepackage{algorithm}

\usepackage{eso-pic}
\usepackage{xspace}
\makeatletter
\DeclareRobustCommand\onedot{\futurelet\@let@token\@onedot}
\def\@onedot{\ifx\@let@token.\else.\null\fi\xspace}

\makeatother

\def\AAT{{$P \kern-2.0pt P \kern-1.0pt ^T$ }}
\def\ATA{{$P \kern-1.0pt ^T \kern-2.0pt P$ }}
\def\AATe{{X \kern-2.0pt X \kern-1.0pt ^T }}
\def\ATAe{{X \kern-1.0pt ^T \kern-2.0pt X }}

\def\pdd{{p_{\kern-2pt\lower+1pt\hbox{..}}\kern+1pt }}
\def\pdds{{p_{\kern-2pt\lower+1pt\hbox{..}}^{1/2}\kern+1pt}}
\def\sdd{{s_{\kern-2pt\lower+1pt\hbox{..}}\kern+1pt }}
\def\sdds{{s_{\kern-2pt\lower+1pt\hbox{..}}^{1/2}\kern+1pt}}

\def\pdc#1{{p_{\kern-2pt\lower+1pt\hbox{.}\lower+2pt\hbox{\footnotesize $#1$}\kern+1pt }}}
\def\sdc#1{{s_{\kern-2pt\lower+1pt\hbox{.}\lower+2pt\hbox{\footnotesize $#1$}\kern+1pt }}}

\def\prd#1{{p_{\kern-1pt\lower+2pt\hbox{\footnotesize $#1$}\lower+1pt\hbox{.}\kern+1pt }}}
\def\srd#1{{s_{\kern-1pt\lower+2pt\hbox{\footnotesize $#1$}\lower+1pt\hbox{.}\kern+1pt }}}

\def\vv1k{{{\bf v}^1 \cdots {\bf v}^k}}
\def\xxx1n{{{\bf x}_1 \cdots {\bf x}_n}}

\def\Dm1{{D^{-1}}}
\def\Dm12{{D^{-1/2}}}
\def\Dp12{{D^{1/2}}}
\def\Dxm12{{D_r^{-1/2}}}
\def\Dym12{{D_c^{-1/2}}}
\def\Dxp12{{D_r^{1/2}}}
\def\Dyp12{{D_c^{1/2}}}

\def\YYT{{Y \kern-0.0pt Y \kern-1.0pt ^T }}
\def\YTY{{Y \kern-1.0pt ^T \kern-0.0pt Y }}
\def\YT{{Y \kern-1.0pt ^T}}
\def\XT{{X \kern-1.0pt ^T}}
\def\XTX{{X \kern-1.0pt ^T \kern-2.0pt X }}

\def\FFTD {{F \kern-1.0pt F^T \kern-3.0pt D }}

\def\bbox{{\hfill {$\sqcap \kern-6.0pt \lower+2.4pt\hbox{--} \kern+2.7pt $}}}
\def\squa {{{$\sqcap \kern-6.0pt \lower+2.4pt\hbox{--} \kern+2.7pt $}}}

\def\argmin{\mathop{\rm arg\,min}}

\usepackage{dsfont}

\def\n21{{2,1}}
\def\l2{$\ell_{2}$\xspace}
\def\l1{$\ell_{1}$\xspace}
\def\l21{$\ell_{\n21}$\xspace}

\newcounter{objcounter}

\begin{document}

\title{\Large Multi-Task Learning with Prior Information}
\author{Mengyuan Zhang\thanks{mengyuz@clemson.edu, School of Computing, Clemson University.}
\and Kai Liu\thanks{kail@clemson.edu, School of Computing, Clemson University.}}

\date{}

\maketitle


\fancyfoot[R]{\scriptsize{Copyright \textcopyright\ 2023 by SIAM\\
Unauthorized reproduction of this article is prohibited}}





\begin{abstract} \small\baselineskip=9pt Multi-task learning aims to boost the generalization performance of multiple related tasks simultaneously by leveraging information contained in those tasks. In this paper, we propose a multi-task learning framework, where we utilize prior knowledge about the relations between features. We also impose a penalty on the coefficients changing for each specific feature to ensure related tasks have similar coefficients on common features shared among them. In addition, we capture a common set of features via group sparsity. The objective is formulated as a non-smooth convex optimization problem, which can be solved  with various methods, including gradient descent method with fixed stepsize, iterative shrinkage-thresholding algorithm (ISTA) with back-tracking, and its variation -- fast iterative shrinkage-thresholding algorithm (FISTA). In light of the sub-linear convergence rate of the methods aforementioned, we propose an asymptotically linear convergent algorithm with  theoretical guarantee. Empirical experiments on both regression and classification tasks with real-world datasets demonstrate that our proposed algorithms are capable of improving the generalization performance of multiple related tasks.\end{abstract}

\section{Introduction}

 Over past decades, Multi-Task Learning (MTL)~\cite{caruana1997multitask} has attracted great interest and has been applied successfully in various research areas, including computer vision~\cite{quadrianto2010multitask}, health informatics~\cite{malhotra2022multi}, and natural language processing~\cite{clark2019bam}, etc. MTL aims to boost the generalization performance of multiple related tasks simultaneously by leveraging potentially useful information contained in those related tasks. Generally speaking, the objective of a MTL problem can be formulated as:
\begin{equation}\label{eq:mtl}
    \min_\mathbf{W}\sum_{i = 1}^m\|\mathbf{X}_i \mathbf{w}_i - \mathbf{y}_i\|^2,
\end{equation}
where $\mathbf{W} = [\mathbf{w}_1,  \dots, \mathbf{w}_m]$ is the matrix formed by the weight vectors of $m$ tasks, $\mathbf{X}_i$ denotes the input feature matrix for the $i$-th task, $y_i$ is the corresponding response. Based on the above objective with the goal of minimizing the overall error across all the tasks, previous studies have been done to improve the performance by identifying the intrinsic relationships among tasks~\cite{vithayathil2020survey}, such as a common set of features they share, and some outlier tasks which are in fact not related with other tasks. Existing methods focusing on finding common feature sets can be classified into two major categories: explicit parameter sharing, where all the tasks share some common coefficients explicitly~\cite{ando2005framework}, and implicit structure sharing, which captures the shared structure among tasks implicitly by constraining a common low-rank subspace~\cite{negahban2011estimation}. 
Some studies also perform outlier task detection~\cite{kang2011learning} at the same time.

One key observation which is ignored by previous studies is: the prior knowledge regarding different features relationship is not taken into account, which can play an important role in feature selection. For instance, expertise knowledge may indicate two features have a similar influence on the response, therefore the correspondent coefficients should be close to each other, while two features having opposite influences should have significantly different coefficients. Besides, for common features in related tasks, coefficients for the same feature should not  change dramatically across tasks. Such prior information is reasonable and not difficult to obtain when we deal with real-world data where each feature denotes a certain property of the data and  relationships between different features can be inferred without much cost. 

Another shortcoming of previous MTL studies targeting optimizing Eq.(\ref{eq:mtl}) is the slow convergence rate. In short, gradient descent requires $\mathcal{O}(\frac{1}{\epsilon})$ iterations to achieve $\epsilon$ accuracy, while momentum based methods will not exceed $\mathcal{O}(\frac{1}{\sqrt{\epsilon}})$, which is still sub-linear convergence rate. To accelerate the convergence asymptotically and in light of the objective by adding the regularization terms described above, we propose a novel updating algorithm that enjoys a linear convergence rate
 with $\mathcal{O}(log(\frac{1}{\epsilon}))$ iterations. One can see its advantage over its counterparts when $\epsilon$ becomes sufficiently small in terms of fewer iterations to obtain $\epsilon$ precision.
 
 We explicitly list the main contributions of this paper as:
\begin{itemize}
\item We add a regularization term based on prior information to obtain more accurate coefficients for related tasks. We impose an extra constraint on the coefficients' changing for the same features among related tasks, which will lead the coefficients for the same features to change smoothly across similar tasks.
\item We present three methods to optimize the MTL with prior information objective, including the vanilla gradient descent with a fixed stepsize, the iterative shrinkage thresholding algorithm (ISTA) with modified stepsize searching~\cite{beck2009fast}, and a novel algorithm with linear convergence rate, which is proved to have comparable speed with the fast iterative shrinkage thresholding algorithm (FISTA) with backtracking~\cite{beck2009fast} in experiments. 
We also provide the linear convergence rate proof of our proposed algorithm, validating the superiority of our new algorithm in terms of speed accelerating in the setting of MTL with prior information. 
\item We conduct empirical experiments with real-world data, the experiment results demonstrate the effectiveness of our proposed algorithm.
\end{itemize}

The remaining of this paper is organized as follows: In Section \ref{sec:problem} we introduce the problem formulation of our proposed MTL with prior information. In Section \ref{sec:opt}, we present the three methods to solve the optimization problem we formulated. Detailed proof of the convergence rate of our proposed algorithm in the MTL scenario is provided in Section \ref{sec:conv}. In Section \ref{sec:exp} we demonstrate the experimental results on both regression and classification tasks, showing the good performance of our proposed algorithm in MTL.

Before we start to formulate the problem, we first introduce some notations throughout the paper in advance for clarity and simplicity. Scalars, vectors, and matrices are denoted by lower case letters, bold lower case letters, and bold capital letters, respectively. $x_i$ denotes the $i$-th entry of a vector $\mathbf{x}$, $x_{ij}$ denotes the ($i$, $j$)-th entry of a matrix $\mathbf{X}$. For the $i$-th row in a matrix $\mathbf{X}$, we use $\mathbf{x}^i$, and for the $i$-th column we use $\mathbf{x}_i$.
The $l_{p,q}$-norm of a matrix $\mathbf{X}$ is defined as $\|\mathbf{X}\|_{p,q} = (\sum_i ((\sum_j x_{ij}^p)^{1/p})^q)^{1/q}$. The inner product of matrices $\mathbf{X}$ and $\mathbf{Y}$ is denoted as $\langle\mathbf{X}$, $\mathbf{Y}\rangle$. $\| \cdot \|_F$ represents the Frobenius norm. $\mathbf{P}^k$ denotes the $\mathbf{P}$ matrix we obtain in the $k$-th iteration, $\eta_{Pk}$ denotes the stepsize of $\mathbf{P}$ in the $k$-th iteration, $prox()$ denotes the proximal operator. 
\section{Problem Formulation}\label{sec:problem}

In MTL, we are given $m$ learning tasks associated with the input data $\{\mathbf{X}_1, \dots, \mathbf{X}_m\}$ and the corresponding responses $\{\mathbf{y}_1, \dots, \mathbf{y}_m\}$, where $\mathbf{X}_i \in \mathbb{R}^{n_i \times d}$ with each row as a sample, and $\mathbf{y}_i \in \mathbb{R}^{n_i \times 1}$. $d$ is the number of features in each task, and $n_i$ is the number of samples in the $i_{th}$ task. For each task, we aim to learn a vector of coefficients $\mathbf{p}_i$ such that $\mathbf{y}_i \approx \mathbf{X}_i \mathbf{p}_i$, the matrix $\mathbf{P}$ is formed by the m coefficient vectors as $\mathbf{P} = [\mathbf{p}_1,  \dots, \mathbf{p}_m] \in \mathbb{R}^{d \times m}$.

Assume we have some prior knowledge about the relationships between features, we are able to construct a matrix $\mathbf{D}$ to contain such prior information, in this way the regularization term $\|\mathbf{D} \mathbf{P}\|_F^2$ is able to force that the learned coefficients are in accordance with the given prior information. To give a straightforward illustration of how it works, we provide an example as follows: suppose we have some prior knowledge regarding features, say the $i_{th}$ feature and the $j_{th}$ feature have a similar influence on the response, then accordingly the corresponding coefficients should be close, namely $\|\mathbf{p}^i-\mathbf{p}^j\|$ should be small. In practice, there can be various such feature relationship constraints (say $s$), by following the above we can formulate the constraint as:
\begin{align}
    \begin{split}
   & \sum_{t=1}^{s}\|\mathbf{p}^{i(t)}-\mathbf{p}^{j(t)}\|^2= \\
  & \left\| \underbrace{\begin{bmatrix}
d_{11} & \dots  & d_{1d}\\
\vdots & \ddots & \dots\\
d_{s1} & \dots  & d_{sd}
\end{bmatrix}}_{\mathbf{D}} \cdot 
\begin{bmatrix}
p_{11} & \dots  & p_{1m}\\
\vdots & \ddots & \dots\\
p_{d1} & \dots  & p_{dm}
\end{bmatrix}  \right\|_F^2,
    \end{split}
\end{align}
where $i(t)$ and $j(t)$ denote the indices in $t$-th constraint and each row in $\mathbf{D}$ all elements are 0's except a pair of $\{1,-1\}$ indexed by $i(t)$ and $j(t)$.
Furthermore, related tasks sharing a common set of features should have similar coefficients for each feature, thus we can integrate the term $\|\mathbf{p}_i - \mathbf{p}_{i+1}\|^2$ in the objective to ensure the smoothness of  coefficients' changing between two adjacent tasks (such as a temporal task). 

Combined together, our multi-task learning with prior information is formulated as follows:
\begin{align}
\begin{split}
        \min_{\mathbf{P}} & \frac{1}{2}\sum_{i = 1}^m\|\mathbf{X}_i \mathbf{p}_i - \mathbf{y}_i\|^2 + \lambda \|\mathbf{P}\|_{2,1} \\
  &   + \frac{1}{2} \theta \|\mathbf{D} \mathbf{P}\|_F^2
    + \frac{1}{2}\epsilon \sum_{i=1}^{m-1} \|\mathbf{p}_i - \mathbf{p}_{i+1}\|^2,
\label{eq:obj}
\end{split}
\end{align}
where all the parameters $\lambda, \theta, \epsilon$ are nonnegative. Specifically, the first term is the empirical loss, while the following $l_{2,1}$-norm regularization term is based on the group Lasso penalty~\cite{liu2019learning,liu2019visual}, which is applied to the rows of $\mathbf{P}$ to identify a common set of features~\footnote{One can also change the group Lasso to Nuclear norm ($\|\mathbf{P}\|_*=\sum_i\sigma_i(\mathbf{P})$) to obtain the low-rank property of $\mathbf{P}$, the whole general process remains the same except the proximal solution changing to $SVT$ (Singular Value Thresholding) in Eq.~(\ref{eq:p}).}. The last two regularization terms are aiming to obtain a $\mathbf{P}$ that is consistent with given prior information. From the problem formulation, it's easy to see we have the beneficial fact that the smooth part in the objective function is strongly-convex.

\section{Optimization Algorithm}\label{sec:opt}
In this section, we will show how to solve the problem formulated in Eq.(\ref{eq:obj}) with multiple methods.
Obviously, it is a nonsmooth convex problem due to the existence of the group Lasso regularization term. To handle this, we can decompose Eq.(\ref{eq:obj}) into two parts: 
\begin{align}
\begin{split}
    f(\mathbf{P}) & = \frac{1}{2}\sum_{i = 1}^m\|\mathbf{X}_i \mathbf{p}_i  - \mathbf{y}_i\|^2  
     + \frac{1}{2} \theta \|\mathbf{D} \mathbf{P}\|_F^2 \\
    & + \frac{1}{2}\epsilon \sum_{i=1}^{m-1} \|\mathbf{p}_i - \mathbf{p}_{i+1}\|^2,
\end{split}
\label{eq:diff}
\end{align}
\begin{equation}
\begin{aligned}
    g(\mathbf{P}) = \lambda \|\mathbf{P}\|_{2,1},
    \label{eq:nondiff}
\end{aligned}
\end{equation}

thus
\begin{equation}
    F(\mathbf{P}) = f(\mathbf{P}) + g(\mathbf{P}),
    \label{eq:combined}
\end{equation}
where $f(\mathbf{P})$ is a smooth differential strongly-convex function, $g(\mathbf{P})$ is a nondifferential convex function. Without loss of generality, let $s_{min}^i, s_{max}^i, d_{min}, d_{max}$ denote the minimum and maximum singular value of $\mathbf{X}_i^T\mathbf{X}_i$ and $\mathbf{D}^T\mathbf{D}$, respectively, the Lipschitz constant $L_P$ of $f(\mathbf{P})$ can be calculated as $max_i(s_{max}^i) + \theta \cdot d_{max} + 2\epsilon$, and the strongly-convex constant $\sigma_P$ can be calculated as $min_i(s_{min}^i) + \theta \cdot d_{min} + \epsilon$.

Following the basic approximation model in~\cite{zhang2022rethinking}, given the Taylor expansion of $f(\mathbf{P})$ at $(\mathbf{A})$ is 
\begin{equation}
\begin{aligned}
    T_{\mathbf{A}, \eta_P}(\mathbf{P}) = f(\mathbf{A}) + \langle \nabla f(\mathbf{A}), \mathbf{P} - \mathbf{A} \rangle + \frac{\eta_P}{2}\|\mathbf{P} - \mathbf{A}\|_F^2,
\end{aligned}
\label{eq:taylor}
\end{equation}
we can minimize $F(\mathbf{P})$ via minimizing its quadratic approximation $M_{\mathbf{A}, \eta_P}(\mathbf{P})$:
\begin{equation}
\begin{aligned}
    \argmin_{\mathbf{P}} M_{\mathbf{A}, \eta_P}(\mathbf{P}) = \argmin_{\mathbf{P}} T_{\mathbf{A}, \eta_P}(\mathbf{P}) + g(\mathbf{P}),
\end{aligned}
\label{eq:qua}
\end{equation}
which admits a unique minimizer for any $\eta_P > 0$. Moreover, as long as $\eta_P \ge L_p$, then $M_{\mathbf{A}, \eta_P}(\mathbf{P})$ is a majorization function w.r.t $F(\mathbf{P})$, therefore we can leverage majorize-minimization (\textit{\textbf{MM}}) to optimize $\mathbf{P}$.

Therefore we can get the following optimization problem:
\begin{equation}
\begin{aligned}
     \mathbf{P} & = \argmin_\mathbf{P} f(\mathbf{A}) + \langle \nabla f(\mathbf{A}), \mathbf{P} - \mathbf{A} \rangle  
      + \frac{\eta_P}{2}\|\mathbf{P} - \mathbf{A}\|_F^2 + \lambda \|\mathbf{P}\|_{2,1}\\
      & = \argmin_\mathbf{P} \frac{\eta_P}{2}\|\mathbf{P} - (\mathbf{A} - \frac{1}{\eta_P}\nabla f(\mathbf{A}))\|_F^2 + \lambda \|\mathbf{P}\|_{2,1},
\end{aligned}
\label{eq:p}
\end{equation}
which leads to a closed proximal operator of rows in $\mathbf{P}$ with the following closed-form solution~\cite{liu2019spherical}:
\begin{equation}
\begin{aligned}
    & \mathbf{U} = \mathbf{A} - \frac{1}{\eta_P} \nabla f(\mathbf{A}), \\
    & \mathbf{p}^i = prox_{\eta_P}(\mathbf{u^i}) = \max (0, 1- \frac{\lambda}{\eta_P \|\mathbf{u}^i\|})\mathbf{u}^i.
\end{aligned}
\label{eq:pi}
\end{equation}

We propose three methods to solve the problem above, including the vanilla gradient descent method with constant stepsize, ISTA with modified stepsize searching, and an algorithm proposed by us with a linear convergence rate. 

In vanilla gradient descent with constant stepsize, the optimal solution at the $k$-th iteration is obtained by solving the following problem
\begin{equation}
\begin{aligned}
    \mathbf{P}^k = \argmin_\mathbf{P} M_{\mathbf{P}^{k-1}, \eta_P}(\mathbf{P})
\end{aligned}
\label{eq:vanilla}
\end{equation}
with an appropriate stepsize $1/\eta_P$. We can verify that the objective has a sufficient decrease when we set $\eta_P$ as the Lipschitz constant $L_P$ of $f(\mathbf{P})$ with regards to $\mathbf{P}$. The algorithm is summarized in Algorithm \ref{alg:1}. 
\begin{algorithm}[tb]
	\caption{Vanilla Gradient Descent Method with Constant Stepsize}
	\label{alg:1}
	\begin{algorithmic}
		\STATE {\bfseries Input:} $\eta_P = L_P$.
		\STATE {\bfseries Initialization:} $\mathbf{P}^0$
		\REPEAT
		\STATE \textbf{1}. Set $\mathbf{A} = \mathbf{P}^{k-1}$. 
		\STATE \textbf{2}. Calculate $\mathbf{P}^k$ according to Eq.(\ref{eq:pi}).
		\UNTIL{convergence}
	\end{algorithmic}
\end{algorithm}

While it is guaranteed that the objective in Eq.(\ref{eq:obj}) is monotonically non-increasing with vanilla gradient descent, an obvious drawback of using $1/L_P$ as the constant stepsize is it is too small to achieve an optimal result rapidly. To improve it, we can apply ISTA with modified stepsize searching. In the previous work about ISTA with backtracking, $\eta_{P0}> 0$ and $\beta_P > 1$ are initialized randomly and we need to find the smallest non-negative integer $i_k$ such that with $\eta_{P} = \beta_P^{i_k}\eta_{P^{k-1}}$ we have 
\begin{equation}
\begin{aligned}
    F(prox_{\eta_P}(\mathbf{P}^{k-1}))
    \leq M_{\mathbf{P^{k-1}}, \eta_P}(prox_{\eta_P}(\mathbf{P}^{k-1})).
\end{aligned}
\label{eq:step}
\end{equation}
One potential flaw in the ISTA with the conventional backtracking method described above lies in the initialization of $\eta_0$. It is possible that the value of $\eta_0$ at the very first step is already larger than the actual Lipschitz constant $L$, and the starting step size is already too small to have fast convergence. And the stepsize $\eta_k$ searching in the $k_{th}$ iteration always starts from $\eta_{k-1}$, thus there may be a larger stepsize available satisfying the condition Eq.(\ref{eq:step}) in the $k_{th}$ iteration that cannot be discovered by this searching process. For this reason, we do the stepsize searching reversely starting from setting $\eta_0$ to its Lipschitz constant $L$, then we keep shrinking it until Eq.(\ref{eq:step}) is not satisfied in terms of $\mathbf{P}$ in the optimization process, which is 
\begin{equation}
\begin{aligned}
    F(prox_{\eta_P}(\mathbf{P}^{k-1})) > M_{\mathbf{P^{k-1}}, \eta_P}(prox_{\eta_P}(\mathbf{P}^{k-1})),
\end{aligned}
\label{eq:istap}
\end{equation}
in this way, we are able to find the largest stepsize meeting the condition in each iteration.
By updating as Algorithm \ref{alg:2}, the objective is decreasing much faster than vanilla gradient descent with constant step size. 
\begin{algorithm}[tb]
	\caption{ISTA with Modified stepsize Searching}
	\label{alg:2}
	\begin{algorithmic}
		\STATE {\bfseries Input:} $\eta_{P0} = L_P, 0 < \beta_P < 1$.
		\STATE {\bfseries Initialization:} $\mathbf{P}^0$
		\REPEAT
		\STATE \textbf{1}. Find the smallest integer $i_k$ such that 
		with $\eta_{P} = \beta_P^{i_k}\eta_{P0}$ 
		Eq. (\ref{eq:istap}) is satisfied.
    Set $\eta_{Pk} = \eta_{P} / \beta_P$.
    \STATE \textbf{3}. With $\mathbf{A} = \mathbf{P}^{k-1}, \eta_P = \eta_{Pk}$,
		calculate $\mathbf{P}^k$ according to Eq.(\ref{eq:pi}).
		\UNTIL{convergence}
	\end{algorithmic}
\end{algorithm}

While ISTA converges faster than vanilla gradient descent with constant stepsize, the convergence rate is still sub-linear (including FISTA), therefore we propose a new algorithm to solve the multi-task learning problem in Eq.(\ref{eq:p}) with a linear convergence rate, utilizing the strongly-convex property of $f(\mathbf{P})$. As we said before, all the parameters $\lambda, \theta, \epsilon$ are nonnegative, thus we are able to guarantee that $f(\mathbf{P})$ in Eq.(\ref{eq:diff}) is strongly-convex with $\sigma_P$, and Lipschitz smooth with $L_P$.
The algorithm is summarized as Algorithm \ref{alg:3}.
\begin{algorithm}[tb]
	\caption{Fast Algorithm with Linear Convergence rate}
	\label{alg:3}
	\begin{algorithmic}
		\STATE {\bfseries Input:} $\eta_P = L_P, c = \frac{L_P}{\sigma_P}$.
		\STATE {\bfseries Initialization:} $\mathbf{P}^0$, set $\mathbf{A}^0 = \mathbf{P}^0$
		\REPEAT
		\STATE \textbf{1}. 
		calculate $\mathbf{P}^k$ according to Eq.(\ref{eq:pi}).
	\STATE \textbf{2}. Update $\mathbf{A}^{k} = \mathbf{P}^k 
	+ \frac{\sqrt{c} - 1}{\sqrt{c}+1}(\mathbf{P}^k - \mathbf{P}^{k-1})$
		\UNTIL{convergence}
	\end{algorithmic}
\end{algorithm}
By updating as Algorithm \ref{alg:3}, although the objective in Eq.(\ref{eq:obj}) is not guaranteed to be monotonically non-increasing, in general it can achieve an optimal solution with a higher convergence rate compared with Algorithm \ref{alg:1} and Algorithm \ref{alg:2}. 

Figure \ref{fig:compare} shows the objective versus the number of iterations in five algorithms with a synthetic dataset. We can see that ISTA with our modified stepsize searching converges faster than ISTA with backtracking in~\cite{beck2009fast}, and the new algorithm we proposed generally converges faster than FISTA with backtracking in~\cite{beck2009fast}.
\begin{figure}[h]
  \centering
  \includegraphics[width=0.8\linewidth]{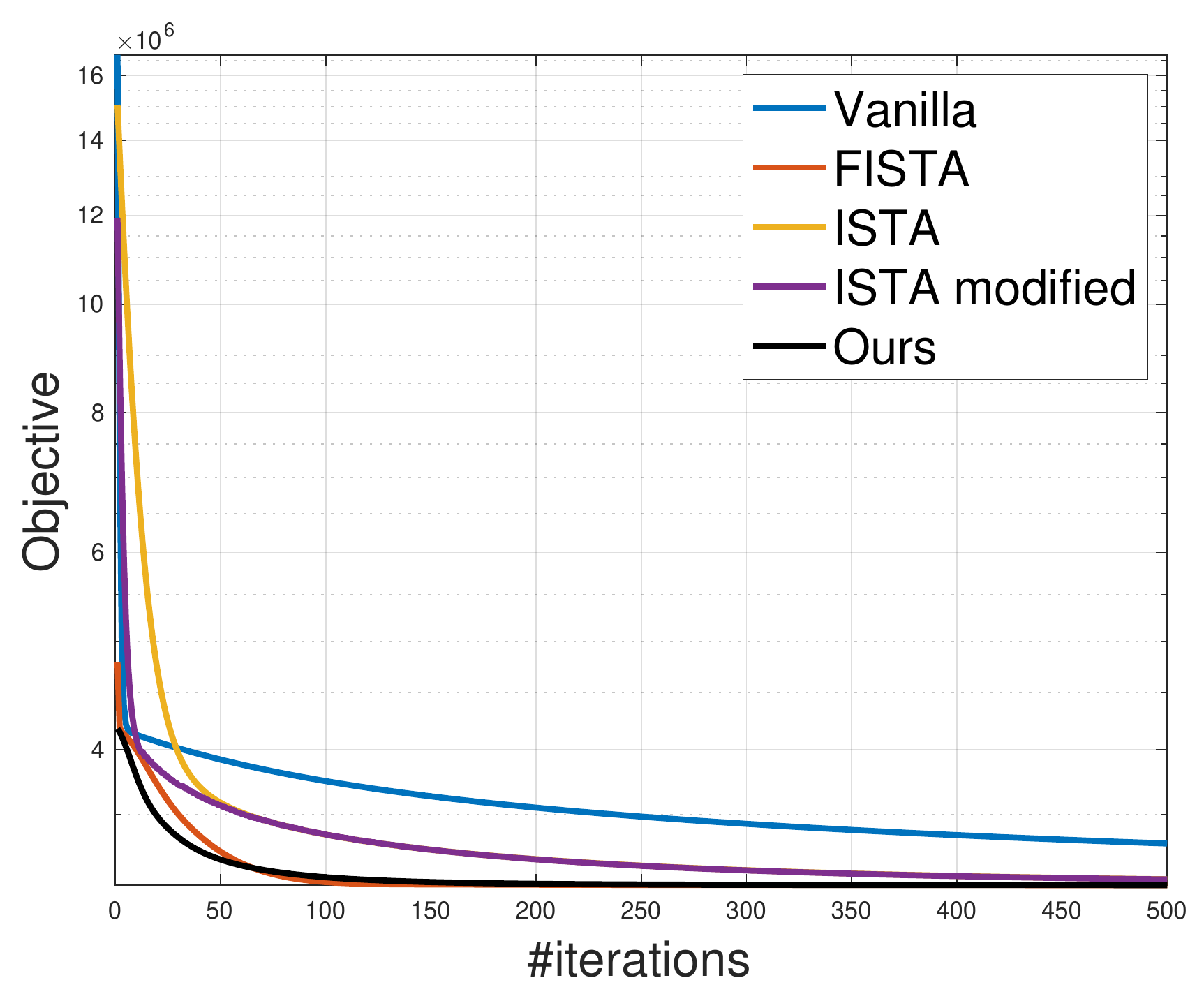}
  \caption{Objective plot in five algorithms.}
  \label{fig:compare}
\end{figure}

\section{Convergence Analysis}\label{sec:conv}
In the previous section, we mentioned Algorithm \ref{alg:2} has a sublinear convergence rate and Algorithm \ref{alg:3} has a linear convergence rate. 
There are other studies showing algorithms with a linear convergence rate for solving such problem~\cite{zhang2013linear}, but different from these studies, there is no strong assumption required in our algorithm, and we utilize the momentum trick following the Nesterov accelerated gradient, which is proven to be unbeatable in general cases. 

The convergence proof of Algorithm \ref{alg:2} can be easily adapted from the proof in~\cite{beck2009fast} to modified step-size searching and be extended from vector variables to matrix variables due to the equivalence of matrix Frobenius norm and vector Euclidean norm. 
Here we only present the key lemma and theorem of the convergence rate:
\begin{lemma}
If for $\mathbf{P} \in \mathbb{R}^{d \times m}$, 
we have
$F(prox_{\eta_P}(\mathbf{P}))
     \leq M_{\mathbf{P},\eta_P}(prox_{\eta_P}(\mathbf{P}))$,
    then for any $\mathbf{A} \in \mathbb{R}^{d \times m}$,
  $F(\mathbf{A}) - F(prox_{\eta_P}(\mathbf{P}))
     \geq \frac{\eta_P}{2} \|prox_{\eta_P}(\mathbf{P}) - \mathbf{P}\|_F^2
    + \eta_P \langle \mathbf{P} - \mathbf{A}, prox_{\eta_P}(\mathbf{P}) - \mathbf{P}\rangle
    $.
\label{lem:1}
\end{lemma}

We present the following theorem about the convergence rate of solving Eq.(\ref{eq:obj}) via Algorithm \ref{alg:2}:
\begin{theorem}
Let $\mathbf{P}^k$ be the output generated by Algorithm \ref{alg:2} in the $k-$th iteration, then for any $k \geq 1$ we have
$F(\mathbf{P}^k) - F(\mathbf{P}^*) = O(\frac{1}{k})$,
where $\mathbf{P}^*$ is the optimal solution in Eq.(\ref{eq:obj}).
\label{thm:ista}
\end{theorem}

Before diving into the detailed proof of convergence rate in Algorithm \ref{alg:3}, we first provide two useful lemmas that are important for the following proof process:

\begin{lemma}
For $F(x) = f(x) + g(x)$, if $g(x)$ is convex, and $f(x)$ is $\sigma-$strongly convex and $L-$smooth, then for any $x, y$ and $\alpha>0$ satisfying 
\begin{align}
\begin{split}
     & f(prox_{\alpha}(y))\\
     \leq 
     & f(y) + \langle \nabla f(y), prox_{\alpha}(y) - y \rangle + \frac{\alpha}{2}\|prox_{\alpha}(y) - y\|^2
\end{split}
\label{eq:lem21}
\end{align}
the following inequality holds:
\begin{align}
\begin{split}
     & F(x) - F(prox_{\alpha}(y)) \\
     \geq 
     & \frac{\alpha}{2}\|x-prox_{\alpha}(y)\|^2
     - \frac{\alpha}{2}\|x - y\|^2 \\
   &  +f(x) - f(y) - \langle \nabla f(y), x - y \rangle.
\end{split}
\label{eq:lem22}
\end{align}
\label{lem:2}
\end{lemma}

\begin{proof}
Consider function $\phi(u) = f(y) + \langle \nabla f(y), u - y \rangle + g(u) + \frac{\alpha}{2}\|u-y\|^2$, it is obvious that such $\phi(u)$ is $\alpha-$strongly convex and $prox_{\alpha}(y) = \argmin_{u}(\phi(u))$.
Then we have
\begin{equation}
\begin{aligned}
    \phi(x) - \phi(prox_{\alpha}(y)) \geq \frac{\alpha}{2}\|x-prox_{\alpha}(y)\|^2.
\end{aligned}
\label{eq:lem23}
\end{equation}
According to Eq.(\ref{eq:lem21}):
\begin{align}
\begin{split}
      & \phi(prox_{\alpha}(y)) \\
      = &  f(y) + \langle \nabla f(y), prox_{\alpha}(y)- y \rangle \\
     & + \frac{\alpha}{2}\|y-prox_{\alpha}(y)\|^2+ g(prox_{\alpha}(y)) \\ 
     \geq & f(prox_{\alpha}(y)) + g(prox_{\alpha}(y))\\
     = & F(prox_{\alpha}(y)),
\label{eq:lem24}
\end{split}
\end{align}
combine with Eq.(\ref{eq:lem23}), we obtain
\begin{equation}
\begin{aligned}
    \phi(x) - F(prox_{\alpha}(y)) \geq \frac{\alpha}{2} \|x - prox_{\alpha}(y)\|^2.
\end{aligned}
\label{eq:lem25}
\end{equation}
Therefore it's easy to get the following inequality after we plug in the formula of $\phi(x)$:
\begin{align}
\begin{split}
     & \phi(x) - F(prox_{\alpha}(y))  \\
 = &   F(x) - f(x) + f(y) 
     + \frac{\alpha}{2}\|x-y\|^2 \\
    & + \langle \nabla f(y), x - y \rangle - F(prox_{\alpha}(y)) \\
 \geq & \frac{\alpha}{2} \|x - prox_{\alpha}(y)\|^2,
\end{split}
\label{eq:lem26}
\end{align}
which is the same as Eq.(\ref{eq:lem22}).
\end{proof}

\begin{lemma}
For any vector $\mathbf{a}, \mathbf{b}$ and constant $\beta < 1$, we have the following equation: 
\begin{equation}
\begin{aligned}
    \|\mathbf{a} + \mathbf{b}\|^2 - \beta\|\mathbf{a}\|^2 = (1-\beta)\|\mathbf{a} + \frac{1}{1-\beta}\mathbf{b}\|^2 - \frac{\beta}{1-\beta} \|\mathbf{b}\|^2.
\end{aligned}
\label{eq:lem31}
\end{equation}
\label{lem:3}
\end{lemma}

Here we introduce the theorem about the convergence rate of Algorithm \ref{alg:3}:
\begin{theorem}
For $F(x) = f(x) + g(x)$ in Eq.(\ref{eq:combined}), $g(x)$ is convex, and $f(x)$ is $\sigma-$strongly convex and $L-$smooth, let $c = \frac{L}{\sigma}$ and $t = \sqrt{c}$. Let $\mathbf{P}^k$ be the $k_{th}$ iteration's output in Algorithm \ref{alg:3}, $\mathbf{P}^*$ be the optimal solution,
$V_k = F(\mathbf{P}^k) - F(\mathbf{P}^*)$. Then for any $k \geq 1$ we have
$V_k \leq (1- \frac{1}{t})^k (V_0 + \frac{\sigma}{2}\|\mathbf{P}^0 - \mathbf{P}^*\|^2)$.
\label{thm:linear}
\end{theorem}

\begin{proof}
According to Lemma \ref{lem:2} and the fact that $f(x)$ is $\sigma-$strongly convex and $L-$smooth, we obtain
\begin{align}
\begin{split}
    &  F(x) - F(prox_{L}(y)) \\
   \geq & \frac{L}{2}\|x-prox_{L}(y)\|^2 - \frac{L}{2}\|x - y\|^2
    + \frac{\sigma}{2}\|x - y\|^2.
\end{split}
\label{eq:t1}
\end{align}
Invoking the above inequality with $x = \frac{1}{t}\mathbf{P}^* + (1 - \frac{1}{t})\mathbf{P}^k$ and $y = \mathbf{A}^k$ in Algorithm \ref{alg:3}, we get
\begin{align}
\begin{split}
    & F(\frac{1}{t}\mathbf{P}^* + (1 - \frac{1}{t})\mathbf{P}^k) - F(\mathbf{P}^{k+1}) \\
    \geq & \frac{L}{2}\|\mathbf{P}^{k+1}-(\frac{1}{t}\mathbf{P}^* + (1 - \frac{1}{t})\mathbf{P}^k)\|^2 \\
    & - \frac{L-\sigma}{2}\|\mathbf{A}^k - (\frac{1}{t}\mathbf{P}^* + (1 - \frac{1}{t})\mathbf{P}^k)\|^2 \\
    = & \frac{L}{2t^2}\|t\mathbf{P}^{k+1} - (\mathbf{P}^* + (t-1)\mathbf{P}^k)\|^2 \\
    & - \frac{L-\sigma}{2t^2}\|t\mathbf{A}^k-(\mathbf{P}^*+(t-1)\mathbf{P}^k)\|^2.
\end{split}
\label{eq:t2}
\end{align}
Since $f$ is a $\sigma-$strongly convex function, for $\alpha \in [0,1]$, we have
$f(\alpha x + (1-\alpha)y) \leq \alpha f(x) + (1-\alpha) f(y) - \frac{\sigma\alpha(1-\alpha)}{2}\|x-y\|^2$, and obviously $\frac{1}{t}\in [0,1]$, so we have
\begin{align}
\begin{split}
     & F(\frac{1}{t}\mathbf{P}^* + (1 - \frac{1}{t})\mathbf{P}^k) \\
     \leq &  \frac{1}{t}F(\mathbf{P}^*) + (1-\frac{1}{t})F(\mathbf{P}^k) 
 - \frac{\sigma t^{-1}(1-t^{-1})}{2}\|\mathbf{P}^k - \mathbf{P}^*\|^2.
\end{split}
\label{eq:t3}
\end{align}
With $V_k = F(\mathbf{P}^k) - F(\mathbf{P}^*)$, we can get 
\begin{align}
\begin{split}
    &  F(\frac{1}{t}\mathbf{P}^* + (1 - \frac{1}{t})\mathbf{P}^k)-F(\mathbf{P}^{k+1}) \\
     \leq &  (1-t^{-1})V_k - V_{k+1} - \frac{\sigma t^{-1}(1-t^{-1})}{2}\|\mathbf{P}^k - \mathbf{P}^*\|^2.
\end{split}
\label{eq:t4}
\end{align}
Combine Eq.(\ref{eq:t4}) and Eq.(\ref{eq:t2}),
we have
\begin{equation}
\begin{aligned}
    \frac{L-\sigma}{2}\|t\mathbf{A}^k - (\mathbf{P}^*+(t-1)\mathbf{P}^k)\|^2 - \frac{\sigma (t-1)}{2}\|\mathbf{P}^k - \mathbf{P}^*\|^2 
     \\ \geq t^2 V_{k+1} -t(t-1)V_k + \frac{L}{2}\|t\mathbf{P}^{k+1}- (\mathbf{P}^*+(t-1)\mathbf{P}^k)\|^2.
\end{aligned}
\label{eq:t5}
\end{equation}
With Lemma \ref{lem:3} and set $a:=\mathbf{P}^k - \mathbf{P}^*, b:= t(\mathbf{A}^k-\mathbf{P}^k), \beta:=\frac{\sigma (t-1)}{L-\sigma}$, then for the left side in the above inequality, we have
\begin{equation}
\begin{aligned}
\small
    & \frac{L-\sigma}{2}\|t\mathbf{A}^k - (\mathbf{P}^*+(t-1)\mathbf{P}^k)\|^2 \\
    & - \frac{\sigma (t-1)}{2}\|\mathbf{P}^k - \mathbf{P}^*\|^2 
    \\
     =&  \frac{L-\sigma}{2}\{\|t(\mathbf{A}^k-\mathbf{P}^k)+(\mathbf{P}^k-\mathbf{P}^*)\|^2 \\
     & - \frac{\sigma (t-1)}{L-\sigma}\|\mathbf{P}^k-\mathbf{P}^*\|^2 \}
    \\
     =& \frac{L-\sigma}{2}\{\frac{L-\sigma t}{L-\sigma}\|(\mathbf{P}^k-\mathbf{P}^*) +
    \frac{L-\sigma}{L-\sigma t}t(\mathbf{A}^k-\mathbf{P}^k)\|^2  \\
   & - \frac{\sigma (t-1)}{L-\sigma t}\|t(\mathbf{A}^k-\mathbf{P}^k)\|^2 \}
    \\
     \leq & \frac{L-\sigma t}{2}\|\mathbf{P}^k - \mathbf{P}^* + \frac{L-\sigma}{L-\sigma t} t(\mathbf{A}^k - \mathbf{P}^k)\|^2
\end{aligned}
\label{eq:t6}
\end{equation}
Therefore we have the following inequality according to Eq.(\ref{eq:t5}):
\begin{equation}
\begin{aligned}
    & t(t-1)V_k + \frac{L-\sigma t}{2}\|\mathbf{P}^k - \mathbf{P}^* + \frac{L-\sigma}{L-\sigma t} t(\mathbf{A}^k - \mathbf{P}^k)\|^2
    \\
    & \geq t^2V_{k+1} + \frac{L}{2}\|t\mathbf{P}^{k+1}- (\mathbf{P}^*+(t-1)\mathbf{P}^k)\|^2.
\end{aligned}
\label{eq:t7}
\end{equation}
With the update rule $\mathbf{A}^{k} = \mathbf{P}^k + \frac{\sqrt{c} - 1}{\sqrt{c}+1}(\mathbf{P}^k - \mathbf{P}^{k-1})$ in Algorithm \ref{alg:3} and $t = \sqrt{c}$, 
\begin{equation}
\begin{aligned}
  & \mathbf{P}^k - \mathbf{P}^* + \frac{L-\sigma}{L-\sigma t} t(\mathbf{A}^k - \mathbf{P}^k)
  \\
  = &\mathbf{P}^k - \mathbf{P}^* + \frac{L-\sigma}{L-\sigma t}\frac{t(t-1)}{t+1}(\mathbf{P}^k - \mathbf{P}^{k-1})
  \\
  = &t\mathbf{P}^k-(\mathbf{P}^*+(t-1)\mathbf{P}^{k-1}).
\end{aligned}
\label{eq:t8}
\end{equation}
Since $L = \sigma t^2$, based on Eq.(\ref{eq:t7}) and Eq.(\ref{eq:t8}), divide both sides of the inequality by $t^2$, we have
\begin{equation}
\begin{aligned}
& V_{k+1} + \frac{\sigma}{2}\|t\mathbf{P}^{k+1}- (\mathbf{P}^*+(t-1)\mathbf{P}^k)\|^2 \\
     \leq & (1-\frac{1}{t})(V_k + \frac{\sigma}{2}\|t\mathbf{P}^k - (\mathbf{P}^* +  (t-1)\mathbf{P}^{k-1})\|^2).
\end{aligned}
\label{eq:t9}
\end{equation}
For $k =0$, with the initialization setting $\mathbf{A}^0 = \mathbf{P}^0$,
\begin{equation}
\begin{aligned}
  \mathbf{P}^0 - \mathbf{P}^* + \frac{L-\sigma}{L-\sigma t} t(\mathbf{A}^0 - \mathbf{P}^0) = \mathbf{P}^0 - \mathbf{P}^*,
\end{aligned}
\label{eq:t10}
\end{equation}
based on Eq.(\ref{eq:t9}), naturally we have
\begin{equation}
\begin{aligned}
 &  V_{k} + \frac{\sigma}{2}\|t\mathbf{P}^{k}- (\mathbf{P}^*+(t-1)\mathbf{P}^{k-1})\|^2 \\
 \leq & (1-\frac{1}{t})^k(V_0 + \frac{\sigma}{2}\|\mathbf{P}^0 - \mathbf{P}^{*}\|^2).
\end{aligned}
\label{eq:t11}
\end{equation}
Thus we can get 
$V_{k} \leq (1-\frac{1}{t})^k(V_0 + \frac{\sigma}{2}\|\mathbf{P}^0 - \mathbf{P}^{*}\|^2)$. The objective in Eq.(\ref{eq:obj}) is not guaranteed to be monotonically non-increasing due to the existence of the $\frac{\sigma}{2}\|t\mathbf{P}^k - (\mathbf{P}^* +  (t-1)\mathbf{P}^{k-1})\|^2$ term, but in general we can expect it can achieve optimal solution with a linear convergence rate. Also, one can see that when the problem is well-conditioned, it converges faster, otherwise, it can be rather slow.
\end{proof}

\section{Experimental Results}\label{sec:exp}
\subsection{Experiment Setup}

 \begin{table*}[ht]
	\begin{center}
		\caption{Regression tasks: generalization performance measures over ten runs}
		\label{tab:1}
		\scalebox{1}{
		\begin{tabular}{*{7}{c}}
			\toprule
			\multicolumn{1}{c}{\textbf{Metric}}  &
			\multicolumn{1}{c}{\textbf{Dataset}}  &
			\multicolumn{1}{c}{\textbf{FedEM}}  &
			\multicolumn{1}{c}{\textbf{DMTL}} &
			\multicolumn{1}{c}{\textbf{MTFL}}  &
			\multicolumn{1}{c}{\textbf{MMTFL}}  &
			\multicolumn{1}{c}{\textbf{OURS-NATURAL}}\\
			\midrule
			\multirow{2}{*}{VE ($\%$)}& School & 38.7$\pm$3.2 & 28.8$\pm$5.6 & 29.7$\pm$2.1 
        & 31.5$\pm$4.2 &\textbf{39.8$\pm$2.2}\\
			 & Sarcos & 51.2$\pm$12.1 & 42.6$\pm$7.3 & 49.2$\pm$5.1  & 49.9$\pm$2.7 & 51.8$\pm$6.6\\
			\midrule
			\multirow{2}{*}{nMSE ($\%$)}& School & 61.2$\pm$0.9 & 75.1$\pm$0.8 & 72.9$\pm$1.1 & 68.7$\pm$3.1  & \textbf{60.2$\pm$0.5}\\
		    & Sarcos & 15.7$\pm$3.1 & 20.9$\pm$0.8 & 19.1$\pm$2.7 & 17.1$\pm$2.2 &14.9$\pm$0.8\\
			 \toprule
			\multicolumn{1}{c}{\textbf{Metric}}  &
			\multicolumn{1}{c}{\textbf{Dataset}}  &
			\multicolumn{1}{c}{\textbf{MTRL}} &
			\multicolumn{1}{c}{\textbf{RMTL}} &
			\multicolumn{1}{c}{\textbf{CLMT}}  &
            \multicolumn{1}{c}{\textbf{MKMTL}}  &
			\multicolumn{1}{c}{\textbf{OURS-ART}}\\
			\midrule
			\multirow{2}{*}{VE ($\%$)}& School & 29.9$\pm$2.0 & 33.6$\pm$5.7 & 37.9$\pm$2.1 & 37.9$\pm$1.9 & \textbf{39.8$\pm$2.5}\\
			 & Sarcos &  42.5$\pm$8.2 & 49.9$\pm$7.2 & 50.1$\pm$9.9  & 50.1$\pm$1.7 & \textbf{51.9$\pm$6.2}\\
			\midrule
			\multirow{2}{*}{nMSE ($\%$)}& School & 73.1$\pm$0.9 & 68.9$\pm$2.7 & 63.7$\pm$0.7 & 62.7$\pm$0.9  & \textbf{60.2$\pm$0.3}\\
		    & Sarcos & 17.9$\pm$0.7 & 15.6$\pm$0.3 & 16.2$\pm$0.6  & 15.7$\pm$0.5 & \textbf{14.7$\pm$0.9}\\
			 \bottomrule
		\end{tabular}}
	\end{center}
\end{table*}

 \begin{table*}[ht]
	\begin{center}
		\caption{Classification tasks: generalization performance measures over ten runs}
		\label{tab:2}
		\scalebox{1}{
		\begin{tabular}{*{7}{c}}
			\toprule
			\multicolumn{1}{c}{\textbf{Metric}}  &
			\multicolumn{1}{c}{\textbf{Dataset}}  &
			\multicolumn{1}{c}{\textbf{FedEM}}  &
			\multicolumn{1}{c}{\textbf{DMTL}} &
			\multicolumn{1}{c}{\textbf{MTFL}}  &
			\multicolumn{1}{c}{\textbf{MMTFL}}  &
			\multicolumn{1}{c}{\textbf{OURS-NATURAL}}\\
	
			\midrule
			\multirow{4}{*}{AUC ($\%$)}& Yale & 96.9$\pm$2.7 & 95.7$\pm$3.2 & 86.8$\pm$1.6  & 92.7$\pm$1.1   & \textbf{97.7$\pm$1.7}\\
			& MNIST & \textbf{98.1$\pm$3.2} & 90.9$\pm$3.1 & 91.2$\pm$1.9   & 91.5$\pm$2.3   & 96.5$\pm$1.2\\
		     & Letter & 63.2$\pm$2.9 & 61.9$\pm$1.8 & 62.1$\pm$2.2    & 61.8$\pm$1.8  & 63.5$\pm$1.1\\
	     & ORL & 81.3$\pm$5.8 & 72.9$\pm$3.7 & 77.2$\pm$5.2  & 77.9$\pm$3.2   & 82.5$\pm$2.2\\
			\toprule
			\multicolumn{1}{c}{\textbf{Metric}}  &
			\multicolumn{1}{c}{\textbf{Dataset}}  &
			\multicolumn{1}{c}{\textbf{MTRL}} &
			\multicolumn{1}{c}{\textbf{RMTL}} &
			\multicolumn{1}{c}{\textbf{CLMT}}  &
            \multicolumn{1}{c}{\textbf{MKMTL}}  &
			\multicolumn{1}{c}{\textbf{OURS-ART}}\\
	
			\midrule
			\multirow{4}{*}{AUC ($\%$)}& Yale  & 96.1$\pm$3.5 & 97.2$\pm$1.9 & 96.7$\pm$2.3   & 95.7$\pm$1.8  & \textbf{97.7$\pm$1.7}\\
			& MNIST  & 93.7$\pm$7.1 & 92.5$\pm$3.1 & 94.7$\pm$3.8   & 93.2$\pm$5.2  & 97.6$\pm$1.2\\
		     & Letter  & 60.3$\pm$2.1 & 62.7$\pm$3.6 & 61.5$\pm$7.2   & 60.9$\pm$7.1  & \textbf{65.5$\pm$1.3}\\
	     & ORL  & 77.6$\pm$3.9 & 80.2$\pm$9.1 & 78.9$\pm$6.3  & 78.8$\pm$3.2  & \textbf{84.3$\pm$1.8}\\
			\bottomrule
		\end{tabular}}
	\end{center}
\end{table*}

	

\begin{figure*}
\centering     
\subfigure[Yale]{\label{fig:11}\includegraphics[width=55mm]{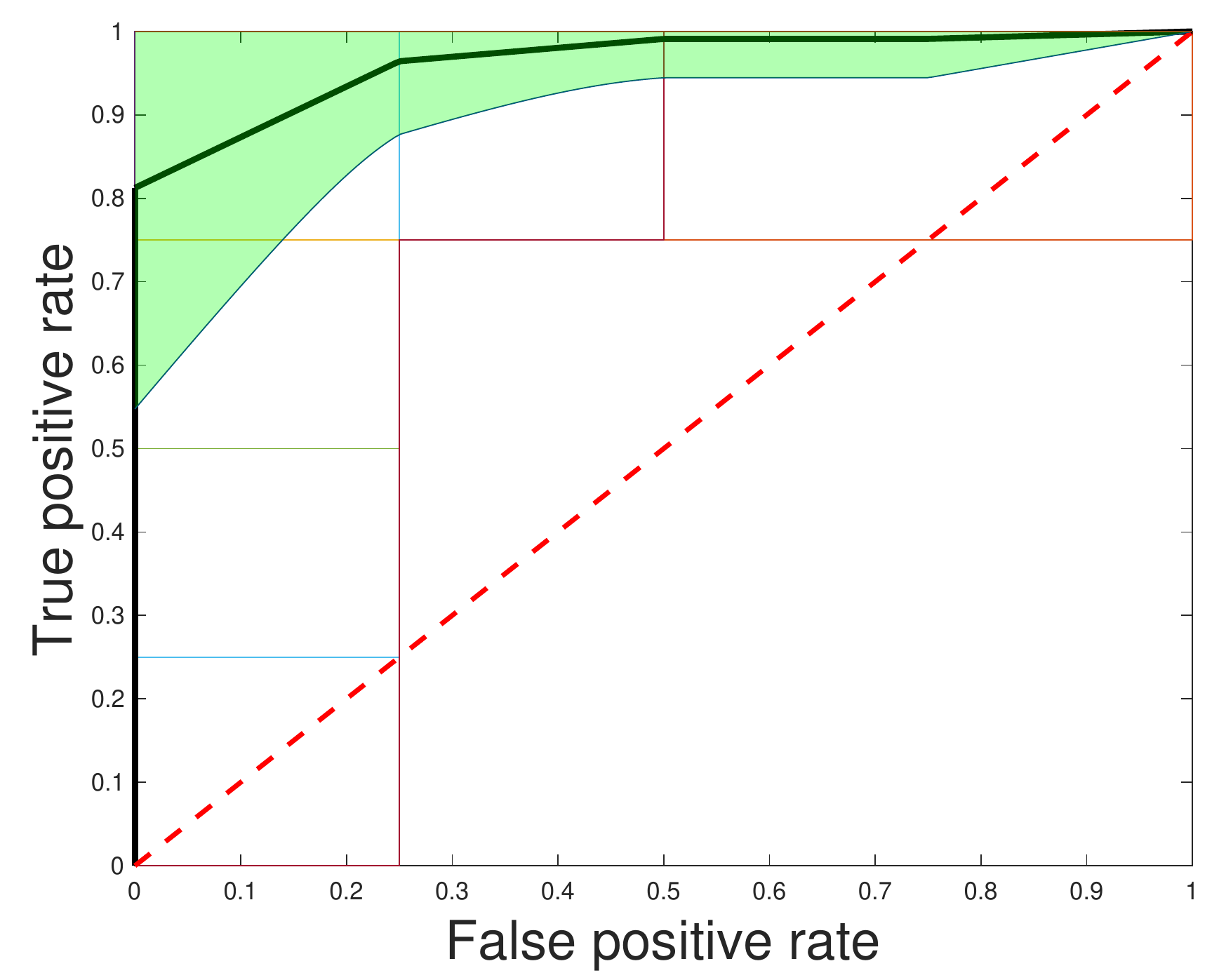}}
\subfigure[MNIST]{\label{fig:12}\includegraphics[width=55mm]{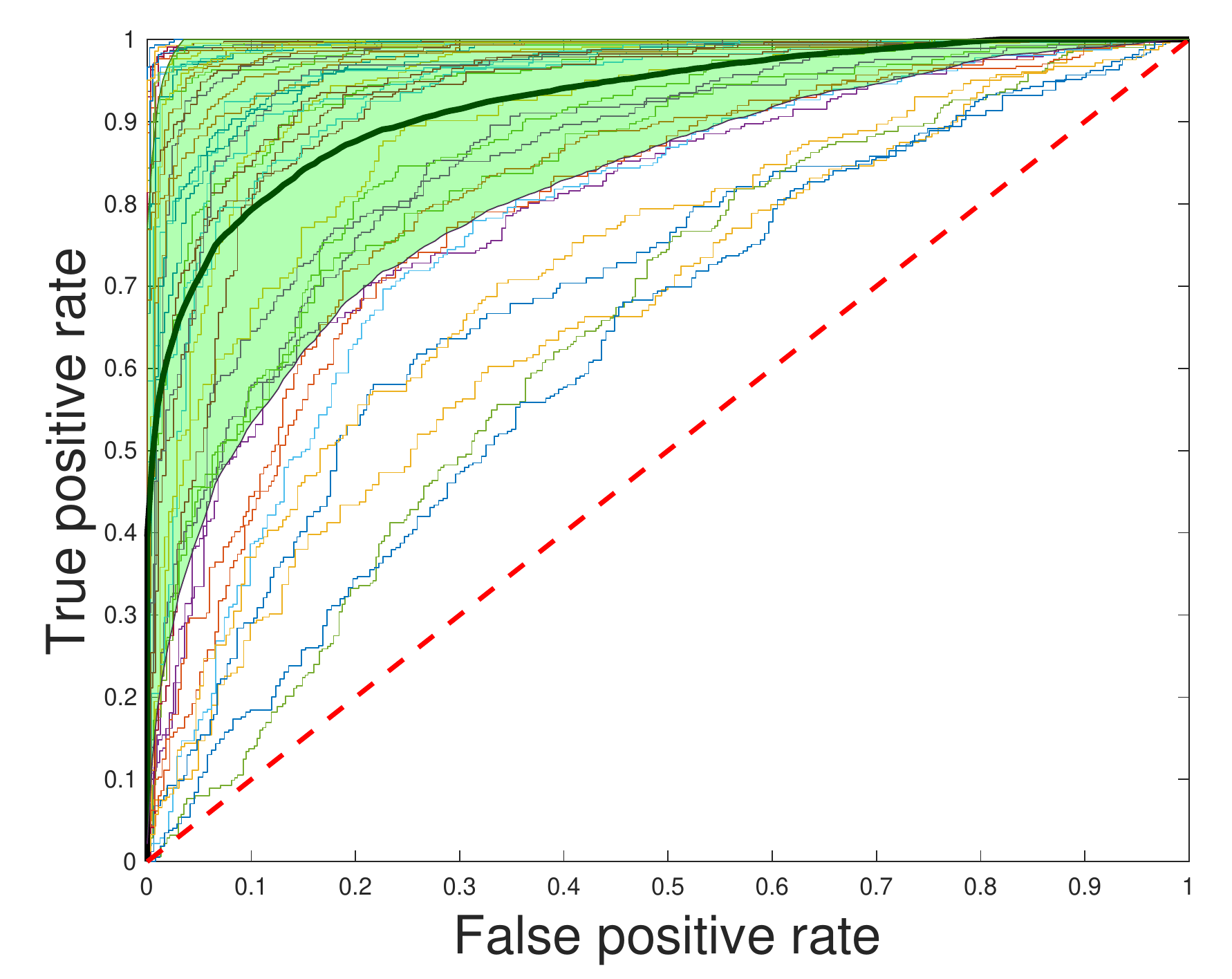}}
\subfigure[ORL]{\label{fig:13}\includegraphics[width=55mm]{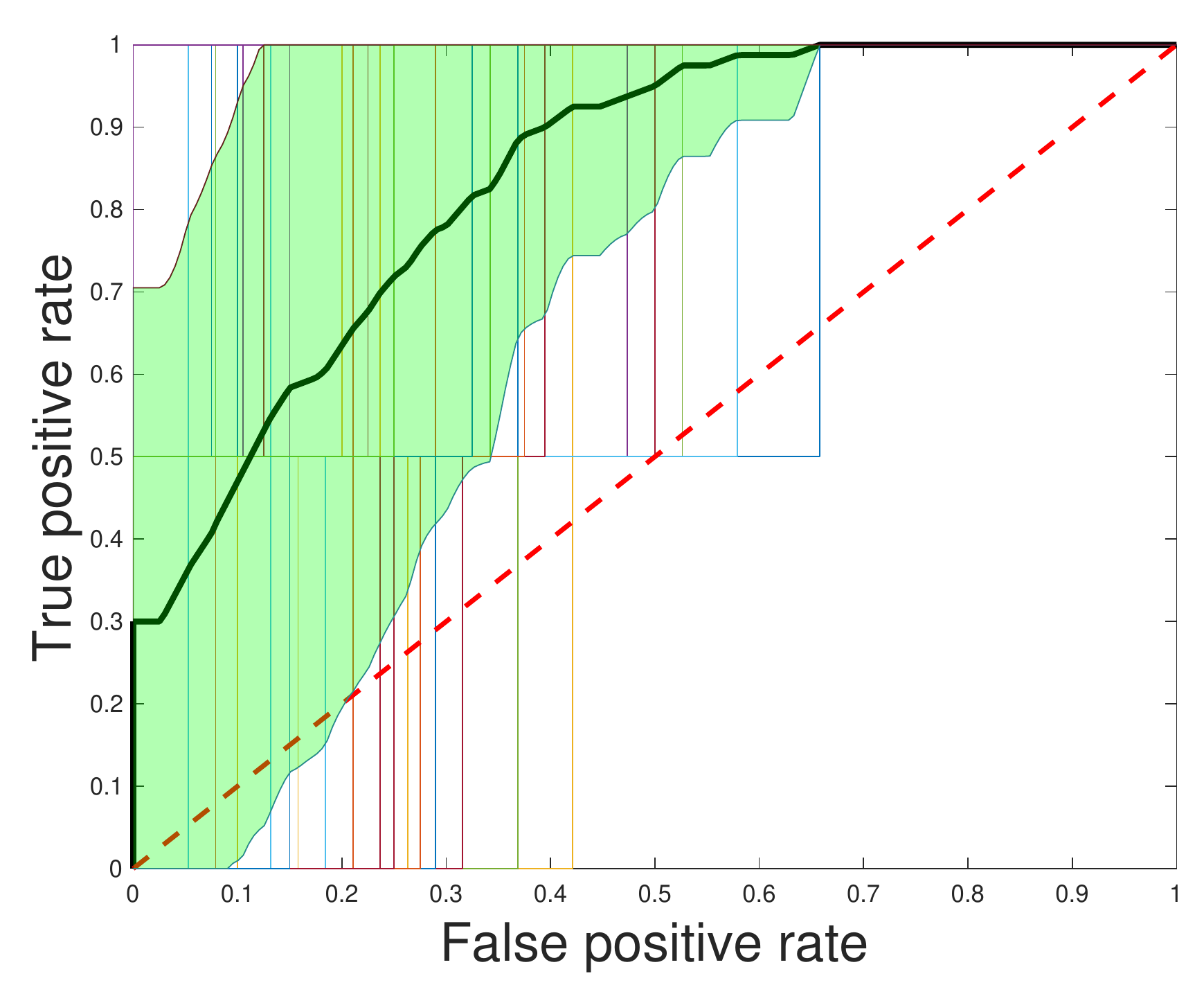}}
\caption{The ROC curve for Yale, MNIST, ORL dataset with tuned parameters.}
\label{fig:roc}
\end{figure*}

\begin{figure*}[ht]
	\centering 
\subfigure[AUC, Yale]{\label{fig:21}\includegraphics[width=41mm]{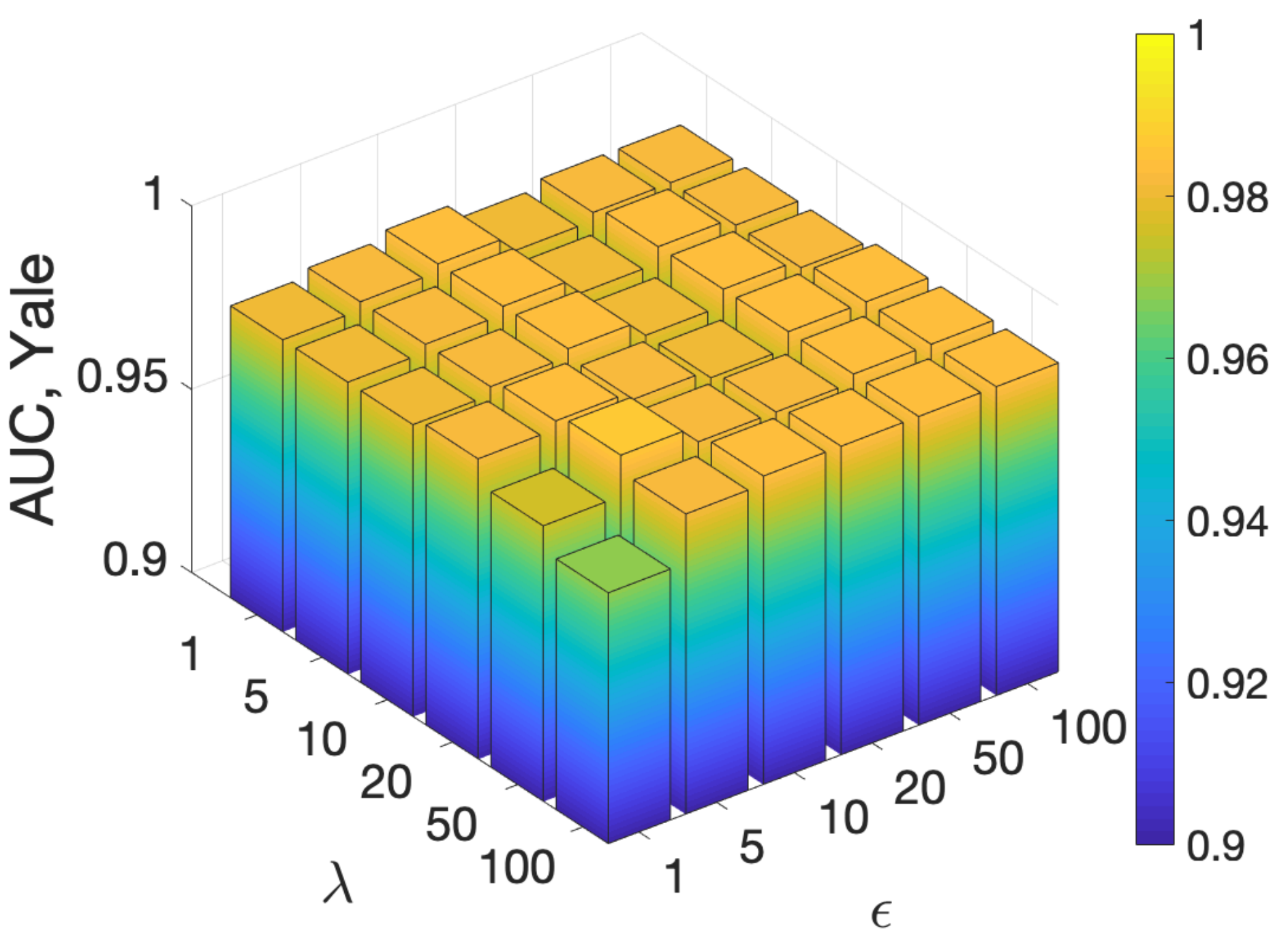}}
\subfigure[AUC, ORL]{\label{fig:22}\includegraphics[width=41mm]{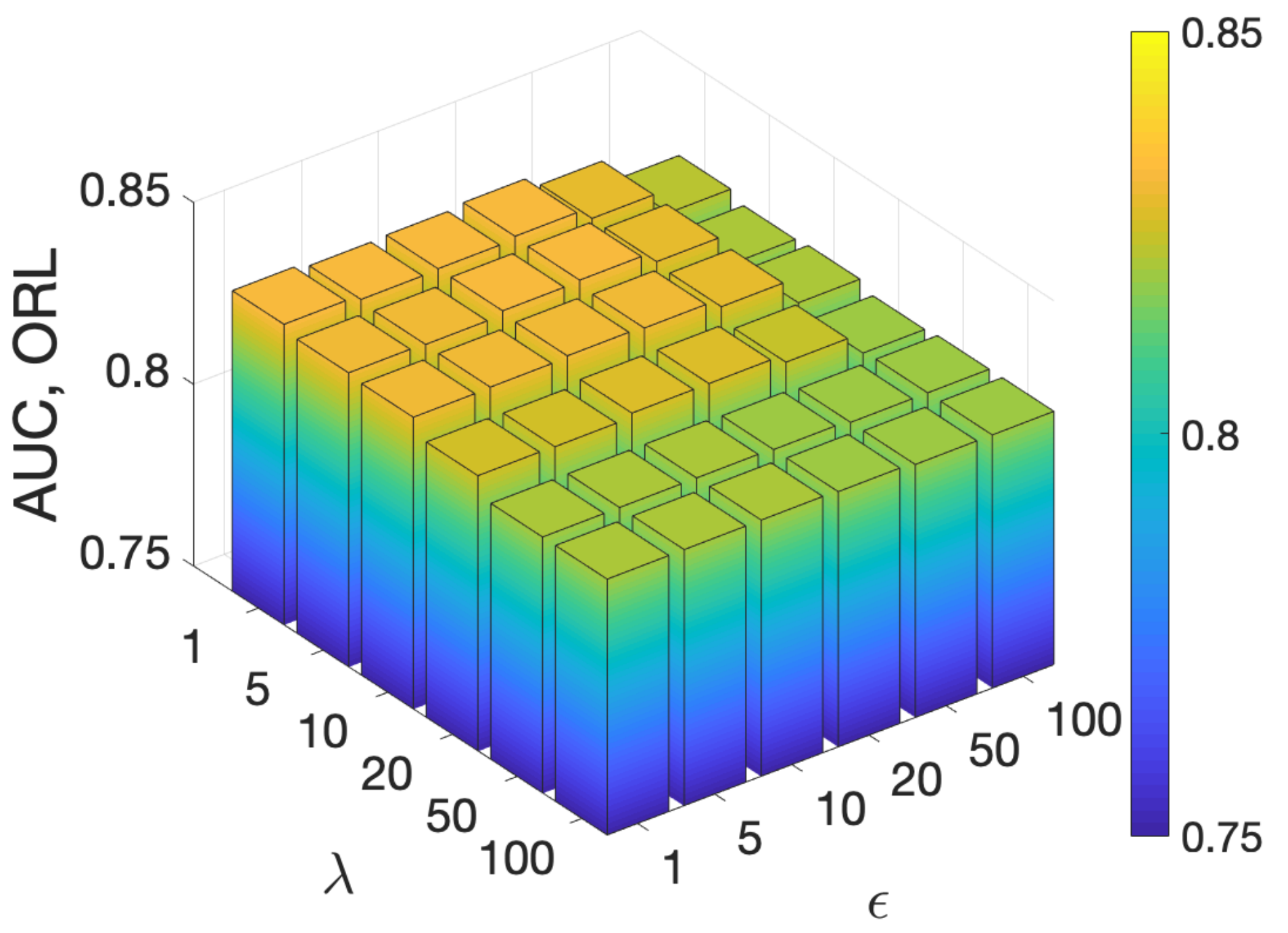}}
\subfigure[VE, School]{\label{fig:23}\includegraphics[width=41mm]{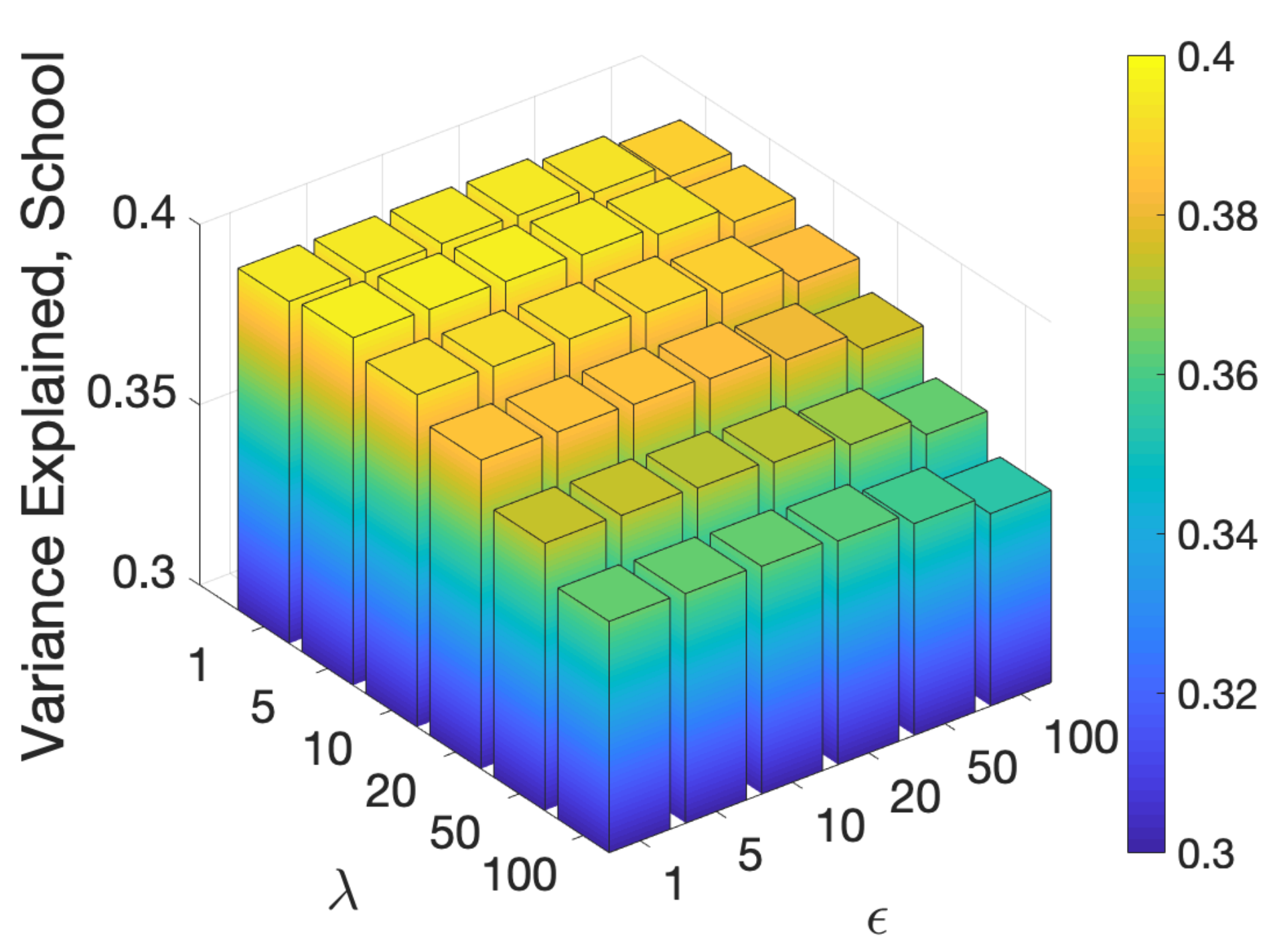}}
\subfigure[nMSE, School]{\label{fig:24}\includegraphics[width=41mm]{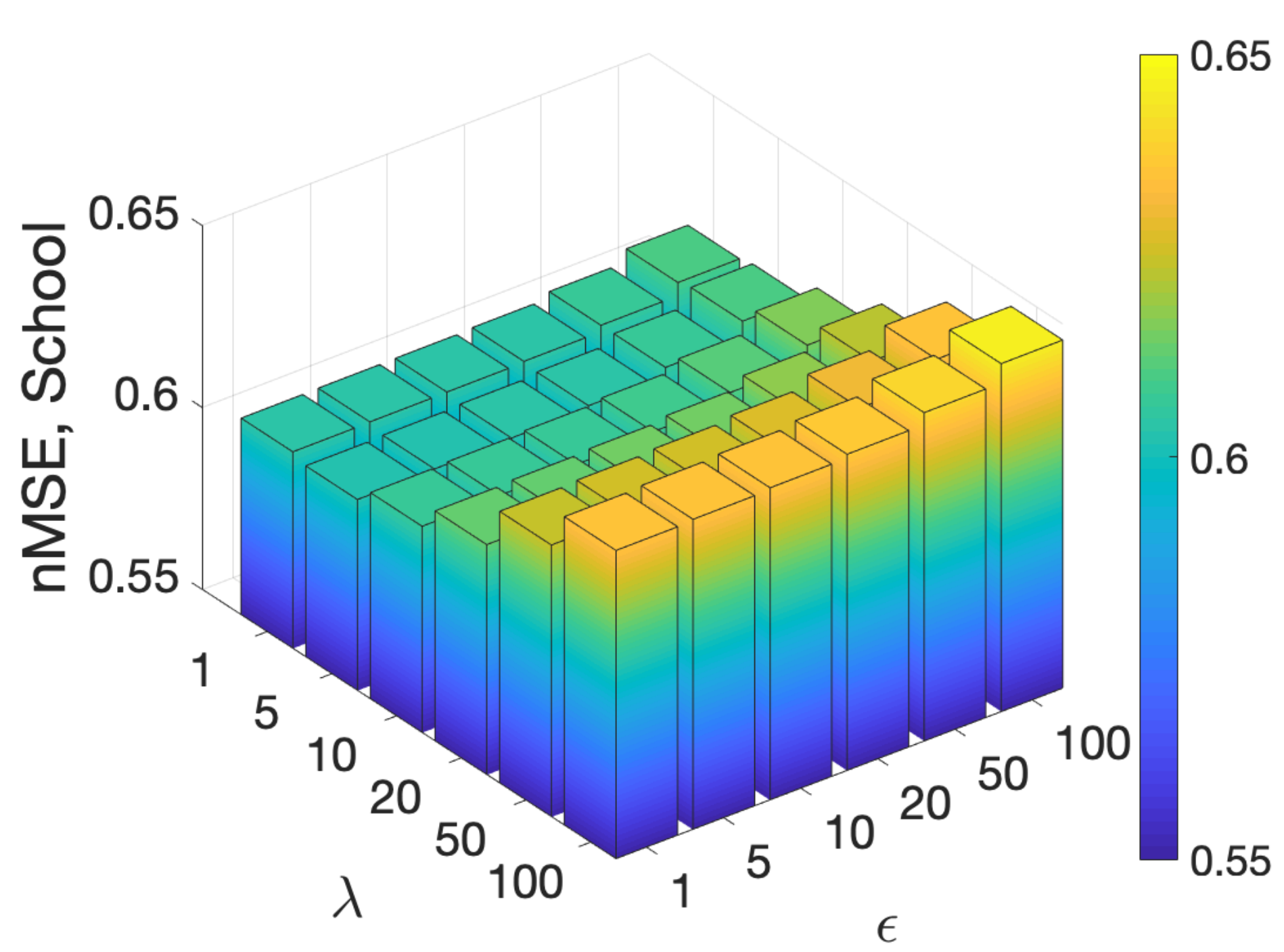}}
	\caption{Ablation study -- the influence of regularization parameters on learning performance.}
	\label{fig:cubic}
\end{figure*}
We compare the experimental results of the following MTL algorithms with our method:
FedEM~\cite{marfoq2021federated},
DMTL~\cite{jalali2010dirty},
MTFL~\cite{jawanpuria2012convex},
MMTFL~\cite{wang2014multiplicative},
MTRL~\cite{zhang2012convex},
RMTL~\cite{chen2011integrating},
CLMT~\cite{ciliberto2015convex},
MKMTL~\cite{liu2018linearized}.
The empirical studies are conducted on the following benchmark multi-task regression and classification datasets:

\textbf{School}~\cite{cortez2008using}: there are exam scores of 15362 students from 139 schools in the dataset, each student is described with 28 attributes. Thus there are 139 related tasks, 
each sample has 28 features along with 1 output. We aim to perform multi-task regression to predict exam scores. 

\textbf{Sarcos}~\cite{vijayakumar2005incremental}: it is collected for an inverse dynamics prediction problem for a seven degrees-of-freedom robot arm. The number of related tasks is 7, and there are 21 features for each sample. Following the work in~\cite{zhang2009semi} we sample 2000 random samples for each task. 

\textbf{Yale}: it contains 165 images from 15 subjects, each image is scaled to 32 $\times$ 32 pixels. We use the first 8 subjects from it to construct related tasks, each task is defined as a binary classification problem of classifying two subjects, there are 28 binary classification tasks. 

\textbf{MNIST}~\cite{lecun-mnisthandwrittendigit-2010}: we use a subset from the MNIST dataset with 10000 samples of 10 handwritten digits. We cast the multi-task learning as 45 binary classification tasks to classify pairs of digits. 

\textbf{Letter}~\cite{Dua:2019}: it consists of handwritten letters from different writers, we construct 8 binary classification tasks from it to distinguish between pairs of letters. 

\textbf{ORL}~\cite{samariaproceedings}: there are 10 different images of 40 distinct subjects. What's different for the ORL dataset is we construct 40 one-vs-all multi-class classification tasks from it rather than one-vs-one binary classification.

During the experiment process, each dataset is randomly split into a training set and a test set. In all classification tasks, the data is split according to a rough 60$\%$-40$\%$ training-test split ratio, in the School regression task 20 random samples are used for training, and in the Sarcos regression task, 50 random samples are selected for training and the rest for testing. The prior knowledge matrix $\mathbf{D}$ can be initialized based on known prior, our common sense, and the correlation among features obtained with statistical methods. In the experiment, we have two methods to obtain the required $\mathbf{D}$: the first method is to utilize the natural correlation among features, we calculate the covariance matrix for all the features, and select the strongest ones as the information contained in $\mathbf{D}$, we call the prior knowledge matrix obtained in this way the $\mathbf{natural}$ $\mathbf{D}$; the second method is we create the strong correlation among features manually by appending features repeatedly on purpose, for example, we transform the original feature $(x_1,x_2,x_3)$ into $(x_1, x_2, x_3,x_1)$ to enhance the correlation among features (this may introduce multicollinearity, but since we only care about the predictive result, it should be fine), the matrix obtained is called as the $\mathbf{artificial}$ $\mathbf{D}$.
Parameters in all the methods are tuned using 5-fold cross-validation, and for each method, we stop the experiment when the objective change is $< 10^{-3}$. We use Matlab R2019a on a laptop with a 1.4 GHz QuadCore Intel Core i5 processor.

\subsection{Results}

Results of the experiment are presented in Table \ref{tab:1} and Table \ref{tab:2}, for regression tasks and classification tasks respectively, our method with a natural $\mathbf{D}$ is denoted as OURS-NATURAL, with an artificial $\mathbf{D}$ is denoted as OURS-ART. In OURS-ART, we manually repeat about 5\% features to construct $\mathbf{D}$. In the case of regression tasks, we report the variance explained (VE) and the normalized mean squared error (nMSE) following previous studies~\cite{chen2011integrating, jawanpuria2012convex}, whereas the receiver operating characteristic (ROC) curve and the area under the ROC curve (AUC) are employed as the classification performance measurements as used in previous studies~\cite{chen2011integrating, jawanpuria2012convex}. ROC curves show the performance of a binary classification model at all classification thresholds by plotting the true positive rate against the false positive rate. Higher explained variance and AUC indicate better performance, and the opposite for nMSE reported. Each experiment with one specific dataset is repeated 10 times and we report the averaged performance and the standard deviation, the best performances are in \textbf{bold}. 

In Table \ref{tab:1} and Table \ref{tab:2}, our method can achieve the best performance in most cases, while the FedEM method achieves better results with the MNIST dataset. The reason perhaps is that neural networks still have unique advantages when dealing with large-scale complex image datasets, and prior knowledge about features is hard to establish with complex image data. Also, in regression tasks, our method works well with both natural prior knowledge and artificial prior knowledge. In image classification tasks, the improvement in performance with artificial prior knowledge matrices is more significant than that with natural ones, the reason may be due to the fact that the natural correlations between features in regression tasks are stronger than those involved in image classification tasks. 

We present the ROC curves with a natural prior knowledge matrix for the classification tasks on Yale, MNIST, and ORL datasets in Figure \ref{fig:roc}, the slim colorful lines are the ROC curves for each task, and the weighted black line is the averaged ROC curve over all the tasks, and the green area represents the range between the mean $\pm$ standard deviation. While for each dataset there is a small number of tasks with performances that are not so perfect on some level, the overall performance is satisfactory with a pretty high AUC, and the mean $-$ standard deviation curve is almost always over the diagonal line, especially for the one-vs-one classification tasks on Yale and MNIST. 
We also provide Figure \ref{fig:cubic} to illustrate the influence of regularization parameters. With a super wide tuning range for each parameter, which is from 1 to 100, the effect of each parameter on the performance is not significant until the value reaches a threshold, there is a generous range for each parameter to be able to provide stable and great performance.
We can roughly draw the conclusion that for classification tasks, no matter whether it’s one-vs-one binary classification tasks or one-vs-all binary classification tasks, the values of regularization parameters have a marginal influence on the results as long as the value is within a reasonable range. While in a regression situation the performance is much more sensitive to the value of parameters, especially to the group Lasso penalty parameter, which is in accordance with common sense that under-fitting happens with large regularization term parameters. 

\section{Conclusion}
We propose a convex formulation of multi-task learning problem utilizing prior information. A novel optimization algorithm to solve the formulated problem with a linear convergence rate is proposed with theoretical guarantee instead of sub-linear rate of the counterparts. Results on benchmark datasets with both regression and classification tasks demonstrate the effectiveness and advantages of our proposed multi-task learning formulation and algorithm. Identifying outlier tasks using the framework will be left to our future work. 

\bibliography{sdm}
\bibliographystyle{plain}
\end{document}